\documentclass[nohyperref]{article}

\usepackage{microtype}
\usepackage{graphicx}
\usepackage{booktabs} 

\usepackage{caption}
\usepackage{subcaption}

\usepackage{multirow}
\usepackage{xspace}

\usepackage{array}

\usepackage{hyperref}

\usepackage[accepted]{icml2022}

\usepackage{amsmath}
\usepackage{amssymb}
\usepackage{mathtools}
\usepackage{amsthm}

\usepackage[capitalize,noabbrev]{cleveref}

\theoremstyle{plain}
\newtheorem{theorem}{Theorem}[section]

\newtheorem{corollary}[theorem]{Corollary}
\theoremstyle{definition}

\theoremstyle{remark}

\usepackage{soul}
\usepackage{tikz}
\usetikzlibrary{calc}
\usetikzlibrary{decorations.pathmorphing}

\makeatletter

\newcommand{\defhighlighter}[3][]{%
  \tikzset{every highlighter/.style={color=#2, fill opacity=#3, #1}}%
}

\defhighlighter{yellow}{.5}
\newcommand{\highlight@DoHighlight}{
  \fill [ decoration = {random steps, amplitude=1pt, segment length=15pt}
        , outer sep = -15pt, inner sep = 0pt, decorate
        , every highlighter, this highlighter ]
        ($(begin highlight)+(0,8pt)$) rectangle ($(end highlight)+(0,-3pt)$) ;
}

\newcommand{\highlight@BeginHighlight}{
  \coordinate (begin highlight) at (0,0) ;
}

\newcommand{\highlight@EndHighlight}{
  \coordinate (end highlight) at (0,0) ;
}

\newdimen\highlight@previous
\newdimen\highlight@current

\DeclareRobustCommand*\highlight[1][]{%
  \tikzset{this highlighter/.style={#1}}%
  \SOUL@setup
  \def\SOUL@preamble{%
    \begin{tikzpicture}[overlay, remember picture]
      \highlight@BeginHighlight
      \highlight@EndHighlight
    \end{tikzpicture}%
  }%
  \def\SOUL@postamble{%
    \begin{tikzpicture}[overlay, remember picture]
      \highlight@EndHighlight
      \highlight@DoHighlight
    \end{tikzpicture}%
  }%
  \def\SOUL@everyhyphen{%
    \discretionary{%
      \SOUL@setkern\SOUL@hyphkern
      \SOUL@sethyphenchar
      \tikz[overlay, remember picture] \highlight@EndHighlight ;%
    }{%
    }{%
      \SOUL@setkern\SOUL@charkern
    }%
  }%
  \def\SOUL@everyexhyphen##1{%
    \SOUL@setkern\SOUL@hyphkern
    \hbox{##1}%
    \discretionary{%
      \tikz[overlay, remember picture] \highlight@EndHighlight ;%
    }{%
    }{%
      \SOUL@setkern\SOUL@charkern
    }%
  }%
  \def\SOUL@everysyllable{%
    \begin{tikzpicture}[overlay, remember picture]
      \path let \p0 = (begin highlight), \p1 = (0,0) in \pgfextra
        \global\highlight@previous=\y0
        \global\highlight@current =\y1
      \endpgfextra (0,0) ;
      \ifdim\highlight@current < \highlight@previous
        \highlight@DoHighlight
        \highlight@BeginHighlight
      \fi
    \end{tikzpicture}%
    \the\SOUL@syllable
    \tikz[overlay, remember picture] \highlight@EndHighlight ;%
  }%
  \SOUL@
}

\newcommand\bcolorbox[2]{\leavevmode\bcbaux{#1}#2 \endbcb}
\def\bcbaux#1#2 #3\endbcb{%
  \colorbox{#1}{\strut#2}%
  \ifx\relax#3\relax\def\next{}\else%
    \sethlcolor{#1}%
    \hl{ }%
    \def\next{\bcbaux{#1}#3\endbcb}%
  \fi%
  \next%
}
\makeatletter
\def\SOUL@hlpreamble{%
    \setul{-\ht\strutbox}{\baselineskip}
    \let\SOUL@stcolor\SOUL@hlcolor
    \SOUL@stpreamble
}
\def\SOUL@stpreamble{%
    \dimen@\SOUL@ulthickness
    \dimen@i=-.5ex
    \advance\dimen@i-.5\dimen@
    \let\SOUL@ulcolor\SOUL@stcolor
    \SOUL@ulpreamble
}

\def\@onedot{\ifx\@let@token.\else.\null\fi\xspace}
\DeclareRobustCommand\onedot{\futurelet\@let@token\@onedot}

\usepackage[textsize=tiny]{todonotes}

\renewcommand{\[}{\begin{eqnarray}}
\renewcommand{\]}{\end{eqnarray}}
\newcommand{\R}{\mathbb{R}}
\newcommand{\E}{\mathbb{E}}

\newcommand{\X}{\cal X}
\newcommand{\Y}{\cal Y}
\newcommand{\N}{{\cal N}}
\newcommand{\I}{{\cal I}}

\DeclareMathOperator{\ent}{Ent}

\newcommand{\figref}[1]{Figure~\ref{#1}}
\newcommand{\equref}[1]{Eq\onedot~\eqref{#1}}

\newcommand\norm[1]{\left\lVert#1\right\rVert}

\icmltitlerunning{A Functional Information Perspective on Model Interpretation}

\begin{document}

\twocolumn[
\icmltitle{A Functional Information Perspective on Model Interpretation}



\icmlsetsymbol{equal}{*}

\begin{icmlauthorlist}
\icmlauthor{Itai Gat}{technion}
\icmlauthor{Nitay Calderon}{technion}
\icmlauthor{Roi Reichart}{technion}
\icmlauthor{Tamir Hazan}{technion}
\end{icmlauthorlist}

\icmlaffiliation{technion}{Technion - Israel Institute of Technology}

\icmlcorrespondingauthor{Itai Gat}{itaigat@technion.ac.il}

\icmlkeywords{Machine Learning, ICML}

\vskip 0.3in
]



\printAffiliationsAndNotice{}  

\begin{abstract}
    Contemporary predictive models are hard to interpret as their deep nets exploit  numerous complex relations between input elements. This work suggests a theoretical framework for model interpretability by measuring the contribution of relevant features to the functional entropy of the network with respect to the input. We rely on the log-Sobolev inequality that bounds the functional entropy by the functional Fisher information with respect to the covariance of the data. This provides a principled way to measure the amount of information contribution of a subset of features to the decision function. Through extensive experiments, we show that our method surpasses existing interpretability sampling-based methods on various data signals such as image, text, and audio.
\end{abstract}

\section{Introduction}
Machine learning is ubiquitous these days, and its impact on everyday life is substantial. Supervised learning algorithms are instrumental to autonomous driving~\citep{lavin2016fast, bojarski2016end, luss2019generating}, serving people with disabilities~\citep{tadmor2016learning}, improve hearing aids~\citep{Hurtley645, fedorov2020tinylstms, green2022speech}, and are being extensively used in medical diagnosis~\cite{deo2015machine, ophir2020deep, zhou2021review}. Unfortunately, since the models that are achieving these advancements are complex, their decisions are usually not well-understood by their operators. Consequently, model interpretability is becoming an important goal in contemporary deep nets research. To facilitate interpretability, in this work, we provide a theoretical framework that allows measuring the information of a decision function by considering the expected ratio between its gradient norm and its value.

Gradient-based approaches are widely used to interpret the model predictions since they bring forth an insight into the internal mechanism of the model~\citep{saliency, gradcam, integrated, deeplift, latentexpl, gradientxinput, guidedbackprop}. These methods produce different explanation maps using the gradient of a class-related output with respect to its input data. While gradient-based visualizations excel at providing class-specific explanations, these explanations could be noisy in part due to local variations in partial derivatives.

In order to overcome such variations, \citet{smoothgrad, adebayo2018sanity} propose to compute the expected output of gradient-based methods with respect to their input. This practice is known as the sampling approach for interpretability methods. Existing sampling-based methods rely on the assumption that the features are uncorrelated, while real-world data such as pixels in images and words in texts are in fact correlated.

We provide a theoretical framework that applies functional entropy as a guiding concept to the amount of information a given deep net holds for a given input with respect to any of the possible labels. The functional entropy \citep{Bakry14}, originated from functional analysis, measures the possibility of change that is encapsulated in any decision function with respect to its input. The functional entropy contrasts the well-known Shannon's entropy that measures the capacity of probability distributions. We relate the functional entropy to the functional Fisher information that measures the amount of expansion in the decision function through its gradient norm with respect to its distribution. The functional Fisher information is defined for non-negative functions in contrast to the Fisher information, which is defined over probability density functions. 

With these mathematical concepts, we are able to construct a sampling framework for gradient-based explanation methods. The connection between functional entropy and functional Fisher information signifies the importance of the covariance encapsulated in the data. It also allows us to seamlessly extract information of a subset of features from their correlated distribution. 

The remainder of the work is organized as follows: In Sec.~\ref{sec:background} we present the notions of functional entropy and functional Fisher information. Next, in Sec.~\ref{sec:method}, we present the role of the data covariance matrix when relating these two functional analysis operators. We use the data correlations to extract the relevant importance of a subset of the input. Lastly, in Sec.~\ref{sec:experiments}, we demonstrate the effectiveness of our approach on diverse data modalities: vision, text, and audio. We study our approach both quantitatively and qualitatively. We visualize our method's scores in the qualitative experiments and compare them to current explainability sampling-based methods with negative perturbations in the quantitative experiments. Our extensive experiments show that, compared to current methods, our method can better identify the features that drive the predictions of the model and quantitatively estimate their importance. In addition, we compare text explanations of sampling-based methods for different types of architectures. This analysis sheds light on the differences between the architectures and, in particular, highlights the difficulties in interpreting transformer-based models.

In summary, our contributions in this work are: (1) We propose a mathematical framework that measures the amount of information in deep nets. Our framework provides information-theoretical grounds for feature attribution sampling-based methods (Theorem~\ref{thm:sobolev_dependent}). Our method uses the functional Fisher information and suggests that the covariance of the data should be considered. (2) We present a novel approach for sampling explanations of a subset of features (Theorem~\ref{thm:sobolev_set}). According to our method, one should use dependent conditional sampling for explanations of feature subsets. (3) We conduct extensive experiments to evaluate our method on various data modalities, quantitatively and qualitatively. Our code is available at \url{https://github.com/nitaytech/FunctionalExplanation}.

\section{Related Work}\label{sec:related_work}

Over the years, many gradient-based methods were developed to interpret decisions of deep nets. An explanation of a decision function $f_y$ should associate features with scores reflecting their impact on the decision. Intuitively, the gradients of a decision function with respect to the features, $\nabla f_y(x)$ represent the extent to which perturbations in each feature would change the value of the function, i.e., the impact of the feature $x_i$ on to the output $y$ is then the partial gradient $\nabla f_y(x)_i$. Gradient-based methods are considered white-box since they allow to inspect the internal mechanism of the decision function. In contrast, black-box methods like LIME~\cite{lime} are not based on gradients but instead interpret the model's prediction for a given data point based on a local linear approximation around this input. Another important black-box method is SHapleyAdditive exPlanations (SHAP)~\cite{shap}, introducing a unified game-theoretic framework for attribution methods based on Shapley values~\cite{shapley}.

\paragraph{Gradient-based methods.} Early gradient-based methods include Saliency maps~\cite{saliency}, GradCAM~\cite{gradcam}, and Guided backpropagation~\cite{guidedbackprop}. These methods use the values of the deep net gradients with respect to the features or the latent space. Interestingly, \citet{nie2018theoretical} show that while the visualizations produced by the guided backpropagation method are impressive, it does not explain well the model predictions. 

Gradient-based methods typically do not account for potential correlations between the features since they estimate the impact of a feature on the prediction using the partial derivative with respect to this feature. A number of methods address this challenge, for example, Gradient$\times$Input~\citep{gradientxinput} and Integrated Gradients~\citep{integrated} multiply the partial derivatives by the input itself. In contrast, we consider the covariance of the data in order to account for potential correlations between the features explicitly. Notably, the integration of the covariance into our score is naturally derived from the functional Fisher information component of our framework.

\paragraph{Sampling-based gradient methods.} There is no guarantee that the gradients of $f_y$ vary smoothly, and they may fluctuate sharply at small scales~\citep{smoothgrad}. Therefore, explanations based on raw gradients are noisy and might emphasize meaningless local variations in the partial derivatives. To overcome this, ~\citet{smoothgrad} proposes SmoothGrad. This method computes $\mathbb{E}_{z\sim\nu}[\nabla f_y(z)_i]$, which is estimated by calculating the average gradients of mutually independent Gaussians $\nu$ centered around the input $x$. Since gradients have signed values, their interpretation is ambiguous. Therefore,~\citet{smoothgrad} proposed the SmoothGradSQ method, which considering the absolute or squared values of the gradients: $\mathbb{E}_{z\sim\nu}[\nabla f_y(z)_i^2]$. Later,~\citet{adebayo2018sanity} proposed VarGrad, estimating the importance of a feature by computing the variance of its partial derivation under Gaussian perturbations. These methods require sampling to compute the expectation, a property that our method shares.

Although gradient-based methods are intuitive, and despite the good visualizations they provide, there is no theoretical framework that explains how they quantify the impact of the features on the decision function. In this work, we introduce such a theoretical framework. Our framework applies the functional entropy as a guiding concept for computing the contribution of each feature to the decision function. Since computing the functional entropy is intractable, we turn to approximation with the functional Fisher information. Our mathematical framework explicitly accounts for correlations between subsets of the data in order to provide a more reliable estimate of the expectations considered by an explanation method. 

\paragraph{Information in deep nets explanation.} Information theory aspects have been studied in deep neural networks. A well-known line of work focuses on the information bottleneck criterion~\cite{information_bottleneck, information2}. This criterion is designed to maximize the mutual information of the latent representation of the input and with the label. Learning to explain approach (L2X, ~\cite{l2x}) takes a different view and uses the mutual information between the input features and each of the labels to select the most informative features for the prediction. These works require to model the joint distribution of features and labels.

Our work uses a functional analysis perspective about information and measures the amount of information in the decision function. Our functional analysis view allows us to avoid modeling the joint probability distributions of features and labels. Recent work by~\citet{removing_bias} propose to maximize functional entropy of multi-modal data in order to remove modal-specific biases and facilitate out-of-distribution generalization. Unlike our work, they aim to improve generalization, and their method does not consider the covariance of the input features.

\section{Background}\label{sec:background}
A discriminative learner constructs a mapping between a data instance $x \in \X$ and a label $y \in \Y$ given training data $S = \{(x_1,y_1),...,(x_m,y_m)\}$. For example, in an object recognition task, $x$ is an image that is composed of pixels. Each pixel consists of a three-dimensional vector that represents the color of the pixel. Generally, $x \in \R^d$ resides in the Euclidean space. Throughout the work, we consider classification tasks for which $y$ is a discrete label. 

A discriminative learning algorithm searches for parameters $w$ to best fit the relation of $(x_i,y_i)$ in the training data. A popular approach is to measure the goodness of fit using the negative log-likelihood loss: the parameters $w$ of a conditional probability model $p_w(y | x)$ are learned while minimizing $-\log p_w(y_i | x_i)$ over the training data. At test time, the class predicted for a test instance $x$ is the one with the highest probability, namely $\arg \max_y p_w(y | x)$. 

The learned probability model assigns a probability for each possible class. In this work, we interpret how the input $x$ relates to each of the possible classes. We suggest to measure the amount of functional information in the data instance $x$ that relates to each of the possible labels $y$. For each class, we relate the learner's preference to a different non-negative function: $f_y(x) = p_w(y | x).$ In the following, we use tools developed in functional analysis (cf.~\citealp{Bakry14}). First, we present the notion of the functional entropy. Then, we connect the functional entropy to the functional Fisher information through the log-Sobolev inequality.

\subsection{Functional Entropy}\label{sec:ent}
Functional entropies~(\citealp{Bakry14}) are defined over a continuous random variable: a function $f_y(z)$ over the Euclidean space $z \in \mathbb{R}^d$ with a Gaussian probability measure $\mu = \N(x,\Sigma)$ whose probability density functions is 
\[
    d \mu(z) = \frac{1}{\sqrt{(2\pi)^d |\Sigma|}}  e^{-{\frac{1}{2}} ((z-x)^\top\Sigma^{-1} (z-x))} dz.
\]
For notational clarity we denote the standard Gaussian measure centered around $x$ by $\nu = \N(x,I)$. It defers from $\mu = \N(x,\Sigma)$ by its covariance matrix. We aim at measuring the functional entropy of a label $y$ and a data instance $x$. Here and throughout, we use $z$ to refer to a stochastic variable, which we integrate over. The functional entropy of the non-negative label function $f_y(z) \ge 0$ is
\[
    \ent_\mu (f_y) &\triangleq& \int_{\R^d} f_y(z)\log \frac{f_y(z)}{\int_{\R^d}f_y(z) d\mu(z) } d \mu(z)
    \label{eq:ent}
\]
We hence define the functional entropy of a deep net with respect to a label $y$ by the function softmax output $f_y(z)$ when $z \sim \mu$ is sampled from a Gaussian distribution around $x$. The functional entropy is non-negative, namely $\ent_\mu(f_y) \ge 0$ and equals  zero only if $f_y(z)$ is a constant. This is in contrast to differential entropy of a continuous random variable with probability density function $q(z)$: $h(q) = -  \int_{\mathbb{R}^d} q(z) \log q(z) d z$, which is defined for $q(z) \ge 0$ with $\int_{\mathbb{R}^d} q(z) d z = 1$ and may be negative.

The functional entropy can be thought of as the Kullback–Leibler (KL) divergence between the prior distribution $p_{\nu}(z)$ and the posterior distribution $q_\nu(z)\triangleq p_{\nu}(z) f_y(z)$ of the decision function $f_y(z) = p(y|z)$ with respect to the data generating distribution over $z$. For clarity, let's assume that $\int_{\R} p_{\nu}(z)f_y(z) dz =  \E_{z\sim\nu}(f_y)=1$. Then, we have: 
\[
    \ent_{\nu}(f_y) &=& \int_{\R} p_{\nu}(z) f_y(z)\log f_y(z) dz 
    \\ &=& \int_{\R} p_{\nu}(z) f_y(z)\log \frac{p_{\nu}(z)f_y(z)}{p_{\nu}(z)} dz 
    \\ &=& \int_{\R} q_{\nu}(z)\log \frac{q_{\nu}(z)}{p_{\nu}(z)} dz 
    \\ &=& \text{KL}(q_\nu(z), p_\nu(z)).
\]
Intuitively, a high KL divergence signifies highly important features, since the learning process forces the the posterior to diverge from the prior for such features.

Unfortunately, the functional entropy is hard to estimate empirically, since it involves the normalized term $\frac{f_y(z)}{\int_{\R^n}f_y(z)~d\mu(z) }$. Since the integral can only be estimated by sampling, the log-scale denominator of its estimate is hard to compute in practice. Next, we use the log-Sobolev inequality to overcome this.
 
\subsection{Functional Fisher Information}\label{sec:lsi}

Instead of estimating the functional entropy directly, we use the log-Sobolev inequality for standard Gaussian measures (cf.~\citealp{Bakry14}, Section 5.1.1). This permits to bound the functional entropy with the functional Fisher information. We denote the functional Fisher information with
\[
    \I_{\nu}(f_y) \triangleq \int_{\R^d} \frac{ \norm{\nabla f_y(z)}^2 }{f_y(z)} d \nu(z)\label{eq:information}
\]
Hereby, $\norm{\nabla f_y(z)}$ is the $\ell_2$-norm of the gradient of $f_y$. The functional Fisher information is a natural extension of the Fisher information, which is defined for probability density functions. Specifically, the log-Sobolev inequality for any non-negative function $f_y(z) \ge 0$ is, 
\[
    \ent_{\nu}(f_y) \le  \frac{1}{2} \I_{\nu}(f_y). \label{eq:lsi}
\]
The above log-Sobolev inequality applies to standard Gaussian distribution $\nu = \N(x,I)$, centered around $x$. Furthermore, the bound is tight for the exponential family (see example in Appendix~\ref{appendix:tight}), which is a desirable property since the entropy is computed with respect to a label $y$ induced by a softmax output. The assumption that the elements the data are independent is limiting in practice, since the information in a data instance $x$ with respect to an interpreted label $y$ consider correlations between the different elements in $x \in \R^d$. In this work we rely on the functional Fisher information of correlated Gaussian distributions $\mu = \N(x,\Sigma)$ and show that taking into account the correlations in $x$ improves the model's explanation of its prediction of $y$.

\begin{figure*}[t!]
    \centering
    \includegraphics[width=\textwidth]{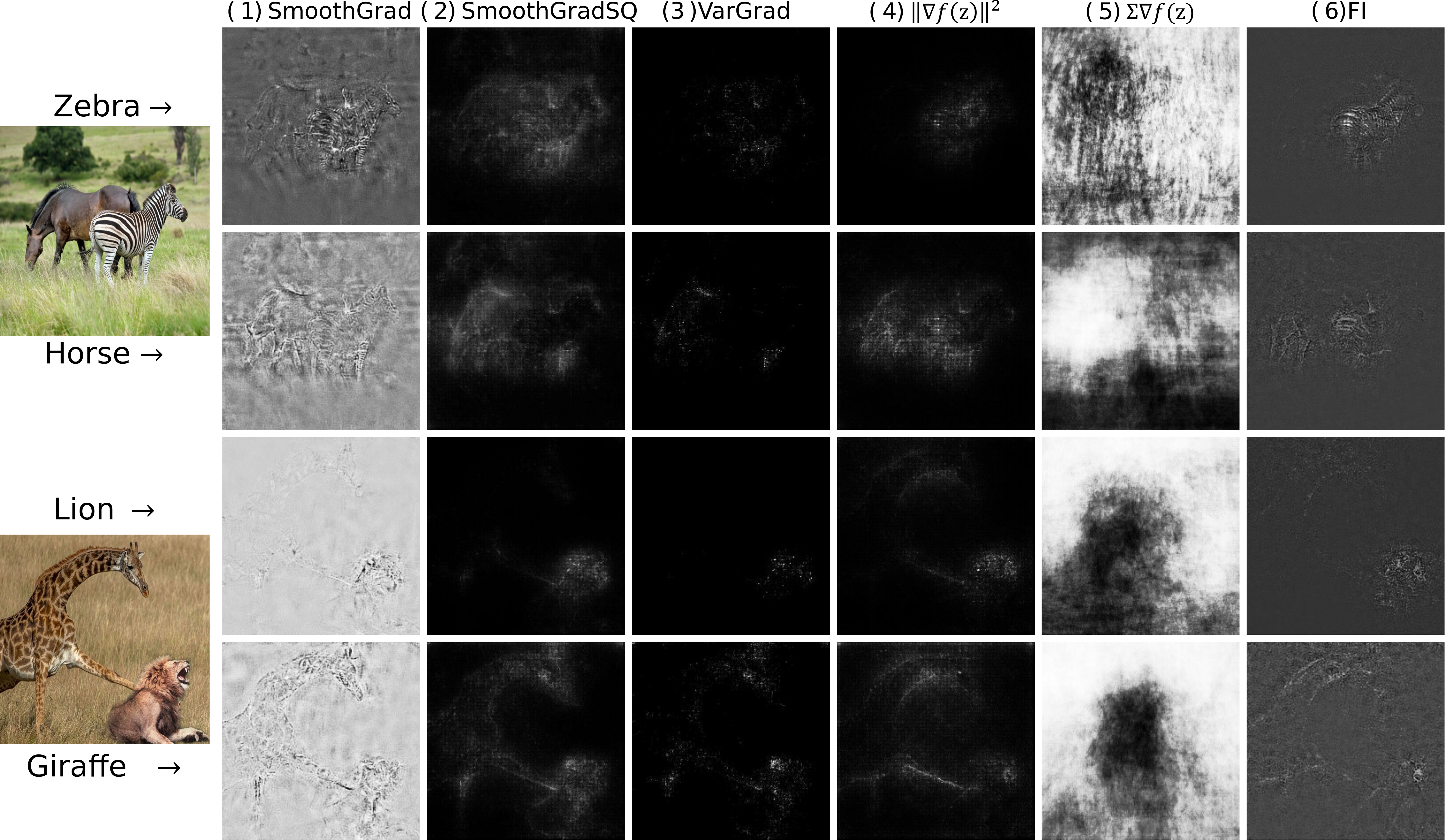}
    \caption{Different explanations of the two possible correct predictions of a fine-tuned ResNet50 network (each image has two animals and hence two correct predictions). We compare our approach (6) with various sampling-based explanation methods: (1) SmoothGrad, (2) SmoothGradSQ and (3) VarGrad. Columns (4) and (5) present the explanations of two components of our method. Particularly, (4) presents the functional information of each pixel when $z$ is sampled from a Gaussian centered around $x$ with the class covariance $\Sigma$, and (5) presents the left hand side of the inner product of Eq.~(\ref{eq:information_def}).}
    \label{fig:image}
\end{figure*}

\section{Feature Contribution via Functional Fisher Information}\label{sec:method}

In this section, we propose a sampling-based method that can quantify the contribution of an input feature $x_i$ to the decision function $f_y$. We start by describing our method, which samples perturbations of the input $x$  from $\nu$, and then expand it to overcome two key challenges. The first is the removal of the independence assumption encoded in $\nu$, and the second is the computation of the contribution of a subset of features when conditioning on the others. We develop two theorems to address those challenges. We first introduce Theorem \ref{thm:sobolev_dependent} which integrates the covariance matrix of $y$ into the functional Fisher information. Then, Corollary~\ref{thm:sobolev_set} yields a sampling protocol suitable for features subsets. This allows us to generate sampling-based explanations from the conditional and marginal distributions.

Recall that the functional entropy is bounded by the functional Fisher information, $\I_{\nu}$. $\I_{\nu}$ measures the expansion of the decision function through the $\ell_2$ norm of the gradient of that function. By considering Eq.~(\ref{eq:information}), we obtain:
\[\label{eq:contrib_total}
    I_{\nu}(f_y) = \sum_{i}{\mathbb{E}_{z\sim\nu}\left[\frac{(\nabla f_y(z)_i)^2}{f_y(z)}\right]}.
\]
Each component in Eq.~(\ref{eq:contrib_total}) can be viewed as the contribution of the feature $x_i$ to the total information.

\subsection{Covariance}

Real-world data features are correlated, and in fact the nature of the correlations changes across classes and modalities. For example, different features are expected to be correlated in images of lions and horses. Likewise, the correlations between words in positive movie reviews are expected to be different from those in negative reviews.  As a result, sampling the stochastic variable $z$ from a Gaussian distribution with a diagonal covariance matrix (i.e., encoding features independence), is incompatible with real-world data. To account for correlations within the data, we propose the following theoretical framework for explaining models applied to data with correlated features.

The functional Fisher information for dependent Gaussian measure is:
\[\label{eq:information_def}
    \I_{\mu}(f_y) \triangleq \int_{\R^d} \frac{ \langle \Sigma \nabla f_y(z), \nabla f_y(z) \rangle }{f_y(z)} d \mu(z).
\]
We next provide a theorem that extend Eq.~(\ref{eq:lsi}) to the correlational case:
\begin{theorem}\label{thm:sobolev_dependent}
    For every non-negative function $f_y: \R^d \rightarrow \R$ and a Gaussian measure $\mu$, 
    \[
            \ent_{\mu}(f_y) \le \frac{1}{2} \I_{\mu}(f_y).
    \]
\end{theorem}

\textit{Proof sketch.} Relying on the log-Sobolev inequality, we perform the following variable change: $z \leftarrow \sqrt{\Sigma}(x+z)$. This permits us to adjust the function's input to a standard Gaussian random variable. By doing so we obtain,
\[
    \ent_{\mu}(f_y(z)) = \ent_{\nu}(g_y(\sqrt{\Sigma}(x+z))).
\]
Then, to conclude the proof, we use integration by substitution of $g_y(z)$ and $f_y(z)$. The full proof is provided in Appendix~\ref{appendix:thm1_proof}.

Theorem~\ref{thm:sobolev_dependent} addresses the independence assumption by considering the covariance of the data. It proposes that when sampling an interpretability method, one should use a dependent Gaussian distribution around a data point and a covariance matrix computed over the data. 

Intuitively, Theorem~\ref{thm:sobolev_dependent} suggests that the contribution of the feature $x_i$ to the total functional information $\I_{\mu}(f_y)$ is:
\[\label{eq:method}
    \int_{\R^d} \Big(\sum_j{\Sigma_{ij}\nabla f_y(z)_j}\Big)\nabla f_y(z)_i d\mu(z).
\]
Note that the left hand side of the multiplication is the sum of the gradients weighted by the class covariances of the feature $x_i$. This weighted sum explicitly accounts for potential correlations between the features and is naturally derived from Theorem~\ref{thm:sobolev_dependent}. To compute this integral we use a standard Monte Carlo sampling procedure. 

\figref{fig:image} demonstrates the importance of each of the components of the integral in Eq~(\ref{eq:method}). Column (4) omits the multiplication by the covariance matrix, while column (5) does not multiply the weighted sum by the gradient. Indeed, it can be seen that the full method explanation in column (6) differs from the explanations in columns (4) and (5).

\subsection{Subset Information}\label{sub:subset}
Multiple scenarios do not necessarily require explaining the entire input but rather only a subset of it. For example, consider the medical imaging task of detecting a tumor in a patient body scan. If the region of interest is only the stomach area, there might be no reason to explain other parts of the body and perturbate the corresponding features. There are other reasons not to perturb a subset of the features, such as when one wants to explain a prediction given fixed values or avoid heavy computations (the covariance matrix has quadratic space complexity).

In this subsection, we present a corollary of Theorem~\ref{thm:sobolev_dependent} to measure the contribution of features to a decision function when conditioning on a fixed value of another subset of the features. As a result, this allows us to measure more appropriately and efficiently the contribution of a desired subset of features or justify conditional sampling when dependencies exist in the data. 

In order to do so, we first partition $x\in \R^d$ into $(x_1, x_2)$ where $x_1 \in \R^{d_1}$ and $x_2\in \R^{d_2}$ ($d_1+d_2=d$). Without loss of generality, let $x_1$ be the set of features for which we are interested in computing the contribution scores. Then, the expectation $x$ and the covariance matrix $\Sigma$ are:
\[
    x = \begin{bmatrix}
        x_{1}\\
        x_{2}\\
    \end{bmatrix},
    \quad
    \Sigma = \begin{bmatrix}
        \Sigma_{11} & \Sigma_{12} \\
        \Sigma_{21} & \Sigma_{22} \\
    \end{bmatrix}.
\]
A known property of the multivariate Gaussian distribution is that any distribution of a subset of variables condition on known values of another subset of variables is also a Gaussian. We denote the conditional distribution of $z_1$ on $z_2$ and the marginal distribution of $z_2$ by $\mu_1, \mu_2$ respectively,
\[
    \mu_1\hspace{-0.3cm}&=&\hspace{-0.3cm}{\cal N} \left(x_1 + \Sigma_{12}\Sigma_{22}^{-1}(z_2-x_2), \Sigma_{11}-\Sigma_{12} \Sigma_{22}^{-1}\Sigma_{21}\right)\nonumber,\\
    \mu_2\hspace{-0.3cm}&=&\hspace{-0.3cm}{\cal N} (x_2, \Sigma_{22}). 
\]
\begin{corollary}\label{thm:sobolev_set}
    For a partitioned input $x=(x_1, x_2)$, a Gaussian measure $\mu$, a conditional distribution $\mu_1$, and a marginal distribution $\mu_2$. For every non-negative function $f_y: \R^d \rightarrow \R$,
    \begin{equation}
        \ent_{\mu}(f_y) \le \frac{1}{2}\E_{z_2\sim \mu_2}\left[ \I_{\mu_1}(f_y\big|z_2) \right].
    \end{equation}
    And, 
    \[
        \ent_{\mu_1}(f_y \big| x_2) \le \frac{1}{2} \I_{\mu_1}(f_y \big| x_2).
    \]
    
\end{corollary}

Our proof relies on Fubini's theorem and on the fact that $\mu_1$ is a Gaussian. The full proof is provided in Appendix~\ref{appendix:thm2_proof}.

Corollary~\ref{thm:sobolev_set} justifies conditional sampling and allows us to measure the contribution of a subset of features more appropriately and efficiently when conditioning on another subset.

\section{Experiments}\label{sec:experiments}

Different data modalities like vision, audio, and language have different characteristics. For example, while linguistics and audio signals are sequential, in images the spatial dimension is prominent. As another example, audio and visual signals are continuous, while text is discrete in nature. The different characteristics of these modalities have led to the development of modal-specific models and data processing pipelines. As a result, an explanation method suitable for models of one modality like vision, might fail when applied to others. This section provides a quantitative and qualitative evaluation of our framework and conducts experiments on the audio, visual, and textual modalities. 

We first describe the tasks, datasets, and the explained modeling architectures we use. Then, we compare our interpretability method to previous sampling-based methods: SmoothGrad, SmoothGradSQ, and VarGrad (see Sec~\ref{sec:related_work}). Consequently, we qualitatively evaluate our framework. The code for our method is in the supplemental material.

\begin{table}[t]
	\centering
	\caption{Quantitative comparison of our proposed method with previous sampling-based explanation approaches. We report AUC values as per negative perturbations evaluation (see text).}
	\begin{scriptsize}
		\begin{tabular}{llccc}
			\toprule
			& &  \multicolumn{2}{c}{Accuracy} &Consistency \\
			\cmidrule(lr){3-4} \cmidrule(lr){5-5}
			Modality & Method & GT & Predicted & Predicted \\
			\midrule
			\multirow{3}{*}{\shortstack[l]{Audio}} &SmoothGrad&37.55&37.38&47.22\\
			&SmoothGradSQ&34.11&33.65&39.40\\
			&VarGrad&39.95&40.65&49.19\\
			\cmidrule[0.15pt]{2-5}
			&Ours&\textbf{51.70}&\textbf{51.09}&\textbf{76.28}\\
			\cmidrule[0.15pt]{1-5}
			\multirow{3}{*}{\shortstack[l]{Image}}&SmoothGrad&46.85&46.64&53.83\\
			&SmoothGradSQ&54.35&51.24&58.35\\
			&VarGrad&47.47&47.14&64.97\\
			\cmidrule[0.15pt]{2-5}
			&Ours&\textbf{57.36}&\textbf{54.48}&\textbf{65.06}\\
			\cmidrule[0.15pt]{1-5}
			\multirow{3}{*}{\shortstack[l]{Text}}&SmoothGrad&62.45&64.24&73.89\\
			&SmoothGradSQ&64.19&64.66&76.26\\
			&VarGrad&\textbf{66.30}&63.45&71.88\\
			\cmidrule[0.15pt]{2-5}
			&Ours&65.72&\textbf{66.30}&\textbf{79.36}\\
			\bottomrule
		\end{tabular}  
	\end{scriptsize}
  \label{table:auc}
\end{table}

\subsection{Experimental Setup}\label{sec:exp_set}

\textbf{Audio.} We use the Google speech commands dataset~\cite{speechcommands}. The dataset consists of 105,829 utterances of 35 words recorded by 2,618 speakers. This data was split into train (80\%), validation (10\%) and test (10\%) sets. We use the M5 model proposed by~\citet{m5} - a CNN architecture followed by an MLP classification layer. This model achieves 73\% accuracy on the validation set. 

\textbf{Vision.}  We use the CIFAR10 to evaluate our method quantitatively~\cite{cifar10}. The dataset is constructed of 50,000 images in the train set and 10,000 images in the validation set. Our model consists of two blocks of a convolution layer followed by a max-pooling layer, and this two blocks are followed by three fully connected layers, with ReLU activation~\cite{rlu}. This model scores 68\% in accuracy on the validation set. For the qualitative results, we use a pre-trained version of ResNet50~\citep{resnet} and fine-tune it on an ImageNet-like animal dataset.~\footnote{\small \url{https://www.kaggle.com/antoreepjana/animals-detection-images-dataset}}

\textbf{Text.} We evaluate our method on the IMDB dataset~\cite{imdb}, with the task of determining the binary sentiment of reviews. The train and test sets consist of 25,000 reviews, and are balanced for sentiment classes. Our model is a BiLSTM~\citep{hochreiter1997long} which feeds a fully connected classification layer, achieving 78\% accuracy on the test set. In addition, in the qualitative evaluation, we present and compare explanations of another three architectures: LSTM, CNN~\citep{zhang2015sensitivity}, and a transformer~\citep{sanh2019distilbert}. Since texts are discrete, we calculate their gradients with respect their embedding.


\begin{figure}[t!]
    \centering
    \includegraphics[width=.48\textwidth]{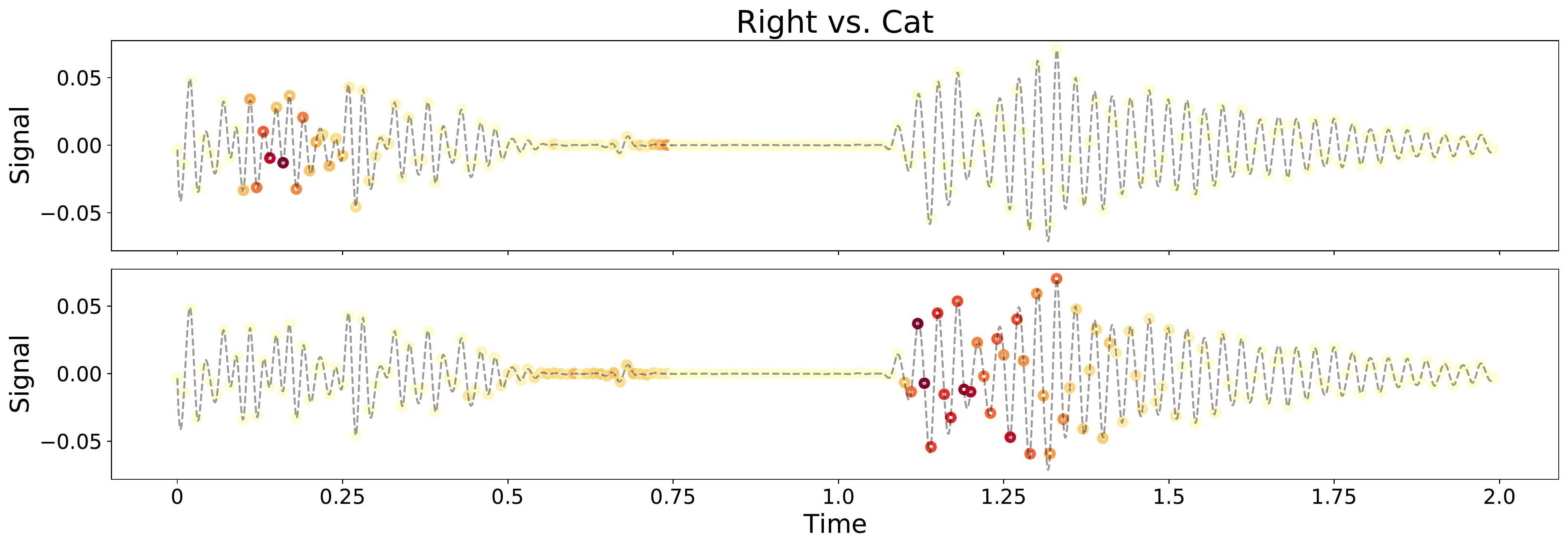}
    \caption{Explanations of the `Right' (top) and `Cat' (bottom) predictions for the utterance of the word `Right' followed by the word `Cat'. Features with high explainability scores are highlighted in color, and warmer colors correspond to higher contributions.}
    \label{fig:audio}
\end{figure}

{
\renewcommand{\arraystretch}{1.5}
\begin{table*}[t!]
    \caption{Explanations generated by our method for predictions of a BiLSTM model trained to predict the sentiment of a movie review. The top two examples are erroneous predictions, while the bottom two are correct. Our method highlights the words that can best explain the model predictions.}
    \begin{center}
    \fboxsep=0pt\relax
    \begin{small}
    \begin{tabular}{ll m{35em}}
        \toprule
        Label & Prediction & Text \\
        \midrule
        Positive & Negative & \bcolorbox{red!15}{ If}\bcolorbox{red!13}{ there}\bcolorbox{red!10}{ is}\bcolorbox{red!10}{ a}\bcolorbox{red!8}{ movie}\bcolorbox{red!10}{ to}\bcolorbox{red!11}{ be}\bcolorbox{red!26}{ called}\bcolorbox{red!33}{ perfect}\bcolorbox{red!2}{ then}\bcolorbox{red!2}{ this}\bcolorbox{red!2}{ is}\bcolorbox{red!3}{ it}\bcolorbox{red!1}{ .}\bcolorbox{red!2}{ So}\bcolorbox{red!68}{ bad}\bcolorbox{red!17}{ it}\bcolorbox{red!26}{ wasn}\bcolorbox{red!18}{ '}\bcolorbox{red!32}{ t}\bcolorbox{red!38}{ intended}\bcolorbox{red!38}{ to}\bcolorbox{red!40}{ be}\bcolorbox{red!43}{ .}\bcolorbox{red!41}{ Do}\bcolorbox{red!75}{ not}\bcolorbox{red!0}{ miss}\bcolorbox{red!0}{ this}\bcolorbox{red!0}{ one}\bcolorbox{red!1}{ !}\\

        Negative & Positive & \bcolorbox{red!2}{ I}\bcolorbox{red!3}{ was}\bcolorbox{red!2}{ very}\bcolorbox{red!2}{ excited}\bcolorbox{red!2}{ about}\bcolorbox{red!12}{ rent}\bcolorbox{red!7}{\#\#ing}\bcolorbox{red!10}{ this}\bcolorbox{red!7}{ movie}\bcolorbox{red!9}{ but}\bcolorbox{red!7}{ really}\bcolorbox{red!8}{ got}\bcolorbox{red!15}{ disappointed}\bcolorbox{red!9}{ when}\bcolorbox{red!7}{ I}\bcolorbox{red!5}{ saw}\bcolorbox{red!3}{ it}\bcolorbox{red!3}{ .}\bcolorbox{red!2}{ The}\bcolorbox{red!3}{ only}\bcolorbox{red!3}{ good}\bcolorbox{red!1}{ thing}\bcolorbox{red!2}{ about}\bcolorbox{red!14}{ it}\bcolorbox{red!16}{ are}\bcolorbox{red!27}{ the}\bcolorbox{red!75}{ great}\bcolorbox{red!27}{ visual}\bcolorbox{red!8}{ effects}\bcolorbox{red!0}{ .}\\

      Positive & Positive & \bcolorbox{red!17}{ Aside}\bcolorbox{red!17}{ from}\bcolorbox{red!17}{ the}\bcolorbox{red!17}{ "}\bcolorbox{red!17}{ Thor}\bcolorbox{red!17}{ "}\bcolorbox{red!16}{ strand}\bcolorbox{red!18}{ of}\bcolorbox{red!16}{ Marvel}\bcolorbox{red!16}{ features}\bcolorbox{red!16}{ ,}\bcolorbox{red!16}{ the}\bcolorbox{red!16}{ "}\bcolorbox{red!16}{ Spider}\bcolorbox{red!15}{\#\#man}\bcolorbox{red!15}{ "}\bcolorbox{red!21}{ stories}\bcolorbox{red!14}{ were}\bcolorbox{red!22}{ always}\bcolorbox{red!18}{ my}\bcolorbox{red!61}{ favourite}\bcolorbox{red!38}{\#\#s}\bcolorbox{red!37}{ .}\bcolorbox{red!35}{ This}\bcolorbox{red!22}{ latest}\bcolorbox{red!29}{ movie}\bcolorbox{red!41}{ is}\bcolorbox{red!32}{ certainly}\bcolorbox{red!33}{ the}\bcolorbox{red!75}{ best}\bcolorbox{red!0}{ one}\bcolorbox{red!1}{ .}\\

    Negative & Negative & \bcolorbox{red!1}{ For}\bcolorbox{red!1}{ me}\bcolorbox{red!1}{ when}\bcolorbox{red!0}{ a}\bcolorbox{red!4}{ plot}\bcolorbox{red!1}{ is}\bcolorbox{red!1}{ based}\bcolorbox{red!0}{ upon}\bcolorbox{red!0}{ a}\bcolorbox{red!0}{ fault}\bcolorbox{red!0}{\#\#y}\bcolorbox{red!0}{ or}\bcolorbox{red!7}{ simply}\bcolorbox{red!75}{ bad}\bcolorbox{red!57}{ premise}\bcolorbox{red!49}{ ,}\bcolorbox{red!40}{ everything}\bcolorbox{red!40}{ that}\bcolorbox{red!39}{ follows}\bcolorbox{red!32}{ is}\bcolorbox{red!14}{ equally}\bcolorbox{red!15}{ fault}\bcolorbox{red!23}{\#\#y}\bcolorbox{red!12}{ and}\bcolorbox{red!21}{ meaning}\bcolorbox{red!13}{\#\#less}\bcolorbox{red!16}{ .}\bcolorbox{red!8}{ It}\bcolorbox{red!14}{ is}\bcolorbox{red!22}{ just}\bcolorbox{red!44}{ empty}\bcolorbox{red!20}{ .}\\
        \bottomrule
    \end{tabular}
    \end{small}
    \end{center}
    \label{tab:qual_text}
\end{table*}
}

\subsection{Quantitative Evaluation}\label{sec:quan} 

The evaluation of interpretation methods is challenging due to the lack of a gold standard. As in previous work, we consider negative perturbations evaluation. Within this framework, we apply two evaluation measures, one of them is novelty of this work. We next describe these measures. 

\paragraph{Negative perturbations.} This procedure is composed of two stages. In the first stage, we use the trained deep net which we would like to explain, and generate an importance score for each feature in the context of each test set example. In the second stage, for each test set example, we mask increasingly large portions of the example (from 10\% up to 90\%, in 10\% steps), in an increasing order of feature importance (from lowest to highest).\footnote{In the audio and image modalities, features correspond directly to the raw input. In the textual modality the features are the word embedding coordinates and the score of each word is the sum of the scores of its coordinates.} At each step, we measure the post-hoc accuracy and post-hoc consistency of each model interpretation method (see below). We eventually plot for each measure its measure value against the masking percentage and report the area under the curve (AUC). Since we mask features from the least to the most important, higher AUC values correspond to better explanation quality. 

For the masking, in the visual setup we replace pixels with black pixels, in the textual setup we replace words with the `unknown' token, and in the audio setup we use zero values.

\paragraph{Post-hoc accuracy.} Following~\citet{gur2021visualization}, in this evaluation setup, we measure the model's accuracy under negative perturbations. We present results when explaining the ground truth label (GT) and the label predicted by the model. The accuracy is measured on the entire test set at each step. Table~\ref{table:auc} (accuracy section) presents the AUC of the accuracy. It demonstrates that our method outperforms previous sampling-based methods on two of the three modalities when explaining the GT and on all modalities when explaining the predictions of the model. Additionally, the differences between our method to the previous methods are higher for the audio and image modalities. This suggests that the covariance has a higher effect on continuous data than discrete data like texts.

\paragraph{Post-hoc consistency.} Since post-hoc accuracy is calculated with respect to the ground truth of each example, it might not highlight the features that drive the predictions of the model, but rather those features that are highly correlated the gold label. For example, consider a model which performs poorly on the test set. A good explanation method should highlight the features which contribute most to the wrong decisions of the model. Post-hoc accuracy measures how masking out features affects the performance of the model, which might not align with our goal. What we would like to understand is how consistent the model is in its predictions before and after the features are masked.

Therefore, we propose to also measure post-hoc consistency. Post-hoc consistency measures the agreement between the predictions of the model when given the masked input with its predictions when given the original unmasked input. To the best of our knowledge, this measure is a novel contribution of this work and we hope it will be adopted by the research community (see similar considerations in \citet{vqarad}). 

Table~\ref{table:auc} (consistency column) presents the AUC values of the post-hoc consistency. Our method, outperforms the previous sampling-based methods in all three modalities. Additionally, the differences between our method and the other methods are bigger than the differences in the post-hoc accuracy for the audio and text modalities. Finally, as discussed in Sec~\ref{sec:method}, without considering the covariance matrix, our method can be viewed as a normalized version of SmoothGradSQ. Since the fluctuations of $f_y(z)$ are small, the differences between the two methods in the post-hoc metrics are minor. To validate this, we evaluated the average Spearman's correlation between the importance scores assigned by SmoothGradSQ, and our method (without the covariance matrix). In the image modality, the correlation is 0.95 (for comparison, it equals 0.13 when the covariance is considered). Hence, the order of the features w.r.t. their scores is almost identical, and the post-hoc values are roughly the same. This observation highlights the importance of the data's covariance.

\subsection{Qualitative Evaluation}\label{sec:qual}

A good explanation method should provide information about the features that are most associated with each of the output classes according to the model. We next examine whether our method provide such information. Following \citet{smoothgrad, adebayo2018sanity, gur2021visualization}, for this analysis we consider examples that have more than one gold label (e.g. consider the Horse and Zebra image in \figref{fig:image}). 

For audio, we concatenate two utterances yielding an utterance with two commands: `Right' and `Cat' (\figref{fig:audio}). We next run our method in order to explain the `Right' and `Cat' predictions of the model (the model can predict additional classes but for the sake of this analysis we are interested in these ``correct'' classes). The highlighted features indicate that our method concentrates on the correct part of the signal, i.e., for `Right' it focuses on the left part of the signal and for `Cat' it focuses on the right part.

For vision, we demonstrate our method's ability to explain the two correct predictions for images that contain two objects. For this end, we use a pre-trained version of ResNet50~\citep{resnet} and fine-tune it on a dataset consisting of 80 animal classes.\footnote{\url{https://www.kaggle.com/antoreepjana/animals-detection-images-dataset}}
We then consider the two images in~\figref{fig:image} and explain the two correct predictions of each image i.e., Horse or Zebra and Lion or Girrafe. As can be seen in~\figref{fig:image}, our method focuses on the object corresponding to the explained prediction. In addition, it reveals strong unique characteristic features of the predicted animals: the lion's mouth, the giraffe's neck and the zebra's stripes. In contrast to the previous sampling-based methods, it seems that our method is capable of distinguishing between the two animals and focus more on the the unique characteristics of each.

For text, Table~\ref{tab:qual_text} presents explanations for the two possible classes of four IMDB reviews. For this analysis we selected for each gold label one review with a correct model prediction and one review with an erroneous prediction.  The stronger the red color is, the higher the score determined by our method. It can be observed that the highlighted words provide intuitive explanations to the predictions of the model.

\begin{figure}[t!]
        \centering
        \begin{subfigure}[b]{.23\textwidth}
         \centering
         \includegraphics[width=1\textwidth]{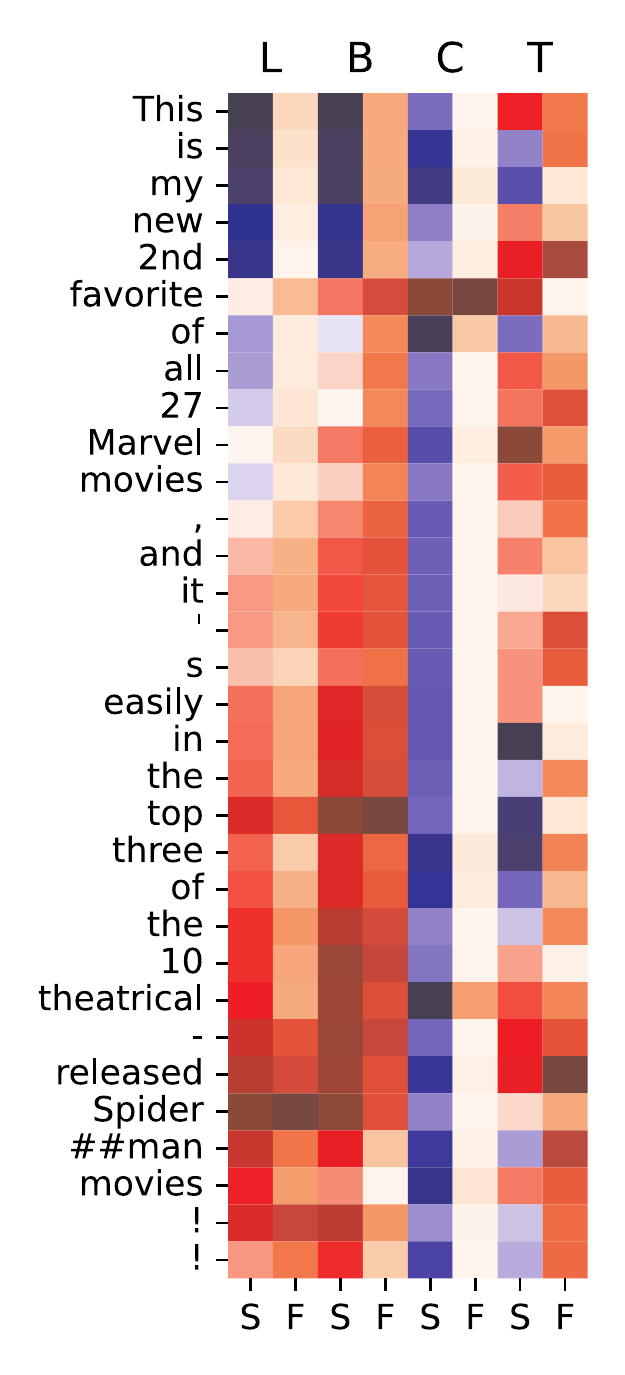}
     \end{subfigure}
     \begin{subfigure}[b]{.23\textwidth}
         \centering
         \includegraphics[width=1\textwidth]{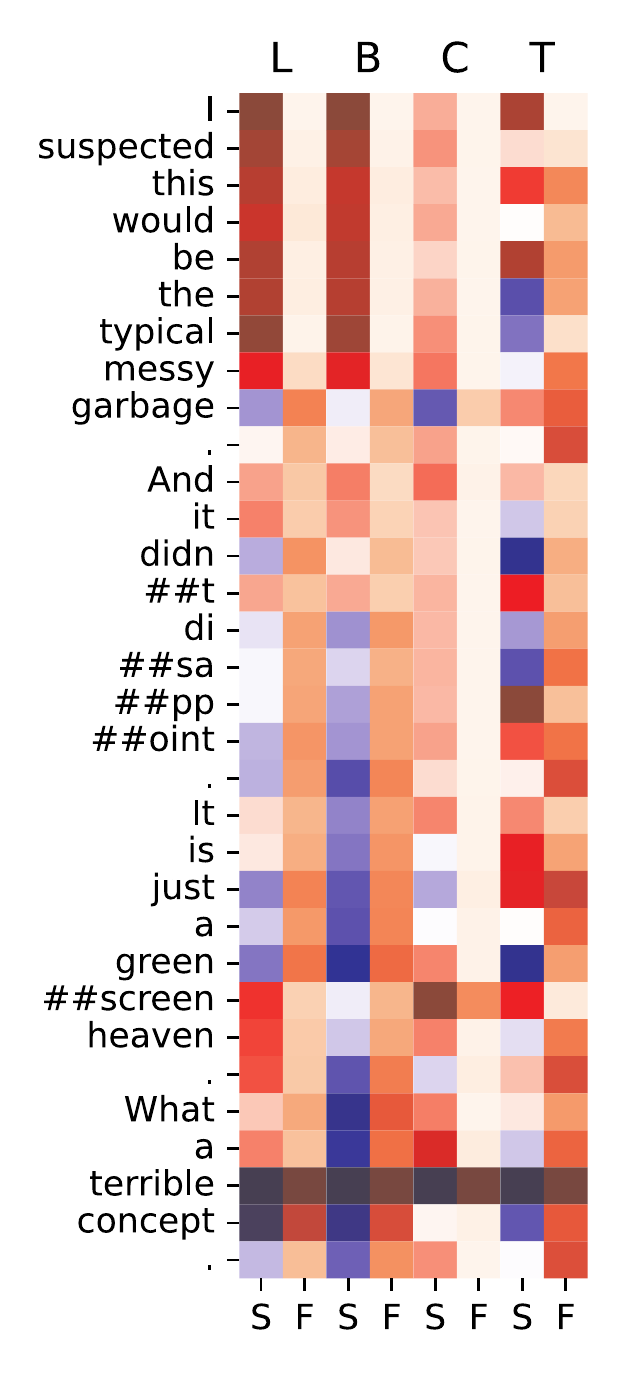}
     \end{subfigure}
    \caption{Explanations generated by our method (F) and SmoothGrad (S) for four architectures: LSTM (L), BiLSTM (B), CNN (C), and transformer (T), for two textual reviews. Red scores represent high scores.}
    \label{fig:heatmaps}
\end{figure}

\paragraph{Explaining text processing architectures.} Unlike for the audio and the image modalities where we use a CNN-based model, for text we use a BiLSTM architecture. However, BiLSTM is not the only architecture that has demonstrated to be effective for text classification, CNNs~\citep{zhang2015sensitivity} and transformers~\citep{bert} are also very prominent. In our last qualitative analysis, we compare text explanations for the following architectures: LSTM (L), BiLSTM (B), CNN (C), and transformer (T). This may shed light on the value of our method for the different architectures. For a fair comparison, we use the vocabulary, tokenizer, and embedding layer of the pre-trained {\sc{distilbert}} model~\cite{sanh2019distilbert} for all architectures. We compare our method (F) to the sign-preserving SmoothGrad method (S).

The heatmaps for two selected reviews are presented in~\figref{fig:heatmaps}. The explanations reveal different patterns of the token contributions to the decisions of each architecture. The sequential architectures, LSTM and BiLSTM, present monotonic patterns of word importance, according to both methods. Additionally, each time the LSTM encounters a negation, contrast, objection or a changing sentiment word, the sign of the SmoothGrad score is changed. A similar pattern is demonstrated for the BiLSTM, however, since it processes the input in both directions (left to right and right to left) it seems to be more robust to these changes.

In contrast, the CNN model is only focused on a small number of relevant words. An intuitive explanation for this behavior is that CNN consists of a funnel of filters. Additionally, our method provides more interpretable explanations than SmoothGrad, highlighting a small number of features with much higher scores than the others. 

Lastly, it seems that the two sampling-based methods fail to explain the transformer architecture, generating very similar scores to all participating features (differences are only in the fifth significance digit). We hypothesise that this pattern stems from the tight connections between the features, as enforced by the transformer attention mechanism.

\section{Conclusion}

Model interpretability plays an essential role in deep nets research. Explanations provide insights about the performance of the model beyond its test-set accuracy and, moreover, they make deep nets accessible to non-experts that can now understand their predictions. Previous sampling-based methods for model interpretability have shown strong empirical results but lack theoretical foundations. In this work we aimed to close this gap and proposed a functional informational viewpoint on sampling-based interpretability methods. We employed the log-Sobolev inequality, which allowed us to compute the functional Fisher information of a set of features. This provides a quantitative feature importance score which takes into account the covariance of the data. We have shown the efficacy of our method both quantitatively and qualitatively. 

In future work we would like to explore the importance of the feature covariance matrix more deeply. For example, while we demonstrated the value of our method for three data modalities, image, text and audio, we would like to better understand in which cases the feature covariance should play a significant rule in explainability. Moreover, since the covariance matrix may impose heavy memory requirements we will aim to find effective and efficient approximations.

\section{Acknowledgements}

Tamir and Itai were funded by Grant No. 2029378 from the United States-Israel Binational Science Foundation (BSF). Roi Reichart and Nitay Calderon were funded by a CornellTech-Techion-AOL grant on learning from multiple modalities - text, vision and video, and by a MOST grant on Natural Interface for Cognitive Robots.

\bibliography{bib}
\bibliographystyle{icml2022}


\newpage
\appendix
\onecolumn

In this appendix, we provide complete proofs for the theorems in the main paper. Code for our method is attached.

\section{Proofs}\label{sec:proofs}
Functional entropies are defined over a continuous random variable, i.e., a function $f_y(z)$ over the Euclidean space $z \in \mathbb{R}^d$ with a Gaussian probability measure $\mu = \N(x,\Sigma)$ whose probability density functions is 
\[
    d \mu(z) = \frac{1}{\sqrt{(2\pi)^d |\Sigma|}}  e^{-{\frac{1}{2}} ((z-x)^\top\Sigma^{-1} (z-x))} dz.
\]
For notational clarity we denote the standard Gaussian measure $\nu = \N(x,I)$, centered around $x$. It defers from $\mu = \N(x,\Sigma)$ by its covariance matrix. We aim at measuring the functional entropy of a label $y$ of a data instance $x$. Here and throughout, we use $z$ to refer to a stochastic variable, which we integrate over. The functional entropy of the non-negative label function $f_y(z) \ge 0$ is
\begin{equation}
    \ent_\mu (f_y) \triangleq \int_{\R^d} f_y(z)\log \frac{f_y(z)}{\int_{\R^d}f_y(z) d\mu(z) } d \mu(z).
    \label{eq:ent_sup}
\end{equation}
We hence define the functional entropy of a deep net with respect to a label $y$ by the function softmax output $f_y(z)$ when $z \sim \mu$ is sampled from a Gaussian distribution around $x$. 
The functional entropy is non-negative, namely $\ent_\mu(f_y) \ge 0$ and equals  zero only if $f_y(z)$ is a constant. This is in contrast to differential entropy of a continuous random variable with probability density function $q(z)$: $h(q) = -  \int_{\mathbb{R}^d} q(z) \log q(z) d z$, which is defined for $q(z) \ge 0$ with $\int_{\mathbb{R}^d} q(z) d z = 1$ and may be negative.

We denote the functional Fisher information with
\[
    \I_{\nu}(f_y) \triangleq \int_{\R^d} \frac{ \norm{\nabla f_y(z)}^2 }{f_y(z)} d \nu(z)\label{eq:information_sup}
\]
Hereby, $\norm{\nabla f_y(z)}$ is the $\ell_2$ norm of the gradient of $f$. The functional Fisher information is a natural extension of the Fisher information, which is defined for probability density functions. Specifically, the log-Sobolev inequality for any non-negative function $f_y(z) \ge 0$ is, 
\[
    \ent_{\nu}(f_y) \le  \frac{1}{2} \I_{\nu}(f_y). \label{eq:lsi_sup}
\]

\subsection{Theorem 1}\label{appendix:thm1_proof}
We derive a log-Sobolev inequality for dependent Gaussian measures. The functional Fisher information for dependent Gaussian measure is
\[
    \I_{\mu}(f_y) \triangleq \int_{\R^d} \frac{ \langle \Sigma \nabla f_y(z), \nabla f_y(z) \rangle }{f_y(z)} d \mu(z).
\]

\begin{theorem}\label{thm:sobolev_dependent_sup}
For every non-negative function $f_y: \R^d \rightarrow \R$ and a Gaussian measure $\mu$, 
\[
        \ent_{\mu}(f_y) \le \frac{1}{2} \I_{\mu}(f_y).
\]
\end{theorem}
\begin{proof}
    We use integration by substitution to adjust the function's input to a standard Gaussian random variable. We denote by $\sqrt{\Sigma}$ the matrix for which $\Sigma = \sqrt{\Sigma}^\top \sqrt{\Sigma}$. Define $\phi(t) = t\log{t}$,
    \begin{equation}
        \ent_{\mu}(f_y) = \int_{\R^d} \phi(f_y(z)) d \mu(z) - \phi (\int_{\R^d} f_y(z) d \mu(z)).
    \end{equation}
    With a change of variable $z \leftarrow \sqrt{\Sigma}(x+z)$, we get $\int_{\R^d} \phi(f_y(z)) d \mu(z) =$
    \begin{equation}
          \int_{\R^d} \frac{e^{-{\frac{1}{2}} (z^\top z)}}{\sqrt{(2\pi)^d}} \phi(f_y(\sqrt{\Sigma}(x+z))) dz.
    \end{equation}
    Therefore, after a change of variables, the distribution changes to a Gaussian distribution with a zero expectation and an identity covariance matrix,
    \begin{equation}
        \int_{\R^d} \phi(f_y(z)) d \mu(z) = \int_{\R^d} \phi(f_y(\sqrt{\Sigma} (x + z))) d \nu(z).
    \end{equation}
    Similarly, for the second term of the functional entropy,
    \begin{equation}
        \phi(\int_{\R^d} f_y(z) d \mu(z)) = \phi(\int_{\R^d} f_y(\sqrt{\Sigma} (x + z)) d \nu(z)).
    \end{equation}
    Consequently,
    \[
        \ent_{\mu}(f_y(z)) = \ent_{\nu}(g_y(\sqrt{\Sigma}(x+z))).
    \]
    With this, we can apply the log-Sobolev inequality for the standard normal distribution,
    \[
        \ent_{\nu}(g_y)  \le  \frac{1}{2} \int_{\mathbb{R}^d} \frac{ \langle \nabla g_y(z), \nabla g_y(z)\rangle}{g_y(z)} d \nu(z). 
    \]
    We conclude the proof by applying the chain rule,
    \[
        \nabla g_y(z) = \sqrt{\Sigma} \nabla f_y(\sqrt{\Sigma}(x+z)).
    \]
    Hence,
    \begin{equation}
         \int_{\mathbb{R}^d} \frac{ \norm{\nabla g_y(z)}^2 }{g_y(z)} d \nu(z) = \int_{\mathbb{R}^d} \frac{\langle \sqrt{\Sigma} \nabla f_y( \sqrt{\Sigma} (x + z)),\sqrt{\Sigma}\nabla f_y( \sqrt{\Sigma} (x + z)) \rangle} {f_y( \sqrt{\Sigma} (x + z))} d \nu(z) \\
    \end{equation}
    Lastly, we apply integration by substitution, thus it equals   
    \begin{equation}
        \frac{1}{2} \int_{\mathbb{R}^d} \frac{\langle \sqrt{\Sigma} \nabla f_y(z),\sqrt{\Sigma}\nabla f_y(z) \rangle} {f_y(z)} d \mu(z).
    \end{equation}

\end{proof}

\subsection{Corollary 2}\label{appendix:thm2_proof}
We partition $x\in \R^d$ into $(x_1, x_2)$ where $x_1 \in \R^{d_1}$ and $x_2\in \R^{d_2}$. Without loss of generality, $x_1$ is the set of features we are interested in. Then, the expectation $x$ and the covariance matrix $\Sigma$ are

\[
    x = \begin{bmatrix}
        x_{1}\\
        x_{2}\\
    \end{bmatrix},
    \quad
    \Sigma = \begin{bmatrix}
        \Sigma_{11} & \Sigma_{12} \\
        \Sigma_{21} & \Sigma_{22} \\
    \end{bmatrix}
\]

\begin{corollary}[Conditional functional Fisher information]\label{thm:sobolev_set_sup}
        For a partitioned input $x=(x_1, x_2)$, a Gaussian measure $\mu$, a conditional distribution $\mu_1$, and a marginal distribution $\mu_2$. For every non-negative function $f_y: \R^d \rightarrow \R$,
    \begin{equation}
        \ent_{\mu}(f_y) \le \frac{1}{2}\E_{z_2\sim \mu_2}\left[ \I_{\mu_1}(f_y\big|z_2) \right]
    \end{equation}
    And, 
    \[
        \ent_{\mu_1}(f_y \big| x_2) \le \frac{1}{2} \I_{\mu_1}(f_y \big| x_2).
    \]\label{eq:cond_int}
\end{corollary}

\begin{proof}
    We extend Theorem~\ref{thm:sobolev_dependent_sup} computing the information of a subset of features. From Theorem~\ref{thm:sobolev_dependent_sup},
    \begin{equation}
        \ent_{\mu}(f_y)\le\frac{1}{2} \int_{\mathbb{R}^d} \frac{\langle \sqrt{\Sigma} \nabla f_y(z),\sqrt{\Sigma}\nabla f_y(z) \rangle} {f_y(z)} d \mu(z).\label{eq:thm_1}
    \end{equation}
    By applying Fubini's theorem on~\equref{eq:thm_1} we get
    \begin{equation}
        \ent_{\mu}(f_y)\le\frac{1}{2}\int_{\R^{d_2}}\int_{\R^{d_1}} \frac{\langle \Sigma\nabla f_y(z)),\nabla f_y(z))\rangle}{f_y(z)} d \mu_1(z_1)~d \mu_2(z_2).\\
    \end{equation}
\end{proof}

\subsection{Log-Sobolev tightness example}\label{appendix:tight}
In the following, we derive the log-Sobolev bound for $f(x) = e^{x}$ to demonstrate the tightness of it for exponential family.
\[
    \ent_{\mu}{e^{x}} &=& \int_{\R^{d}} xe^{x} d\mu(x) - \int_{\R^{d}} e^{x} d\mu(x) \log \left(\int_{\R^{d}} e^{x} d\mu(x)\right)
    \\ &=& (\mu + \sigma^{2})e^{\mu + \frac{1}{2}\sigma^{2}} - (\mu + \frac{1}{2}\sigma^{2})e^{\mu + \frac{1}{2}\sigma^2}
    \\ &=& \frac{1}{2} \sigma^2 e^{\mu+\frac{1}{2}\sigma^2} = \frac{1}{2}\I_{\mu}(e^x).
\]
Hence, the bound hold with equality for $e^x$. This result can be easily extended to a multivariate Gaussian distribution.

\section{Covariance Matrix Calculation}\label{sec:details}

The covariance matrix is a crucial component of our proposed explainability method. 
In order to explain an output class $y$, the covariance matrix of that class $\Sigma$ need to be estimated empirically. 
The covariance matrix may impose heavy memory requirements (the size of the covariance matrix of $d$-dimensional feature vectors is $d^2$). 

In cases of high-dimensional feature vectors, we can partition the features into subsets and sample according to the sampling protocol discussed in Sec.~\ref{sub:subset}.
Alternatively, one may partition the features into subsets and assume that each subset shares the same covariance matrix.
For example, partitioning the features of an image into three subsets, one for each color channel, and assuming all color channels share the same covariance matrix, resulting in a nine times smaller memory usage.

In our experiments, we used shared covariances for vision and text modalities. For the qualitative vision experiments, we used a shared covariance matrix in the size of $H\cdot W\times H\cdot W$ while assuming each color channel share the same covariance matrix. For text, the size of the covariance matrix is $d\times d$, where $d$ is the embedding size, assuming words in the embedding space share the same covariance matrix regardless of their position in a sentence.

Lastly, $\Sigma$ is required to be a positive-definite matrix. In the case where some of the features are constant (e.g., the top row in the MNIST dataset is always black), or when the dimension of the feature vectors is higher than the size of the examples of class $y$, $\Sigma$ will not be a positive-definite matrix. Hence, we suggest adding a small noise to the diagonal of $\Sigma$, which is a well-known practice to modify the matrix to be positive-definite.

\end{document}